\documentclass[letterpaper, 10 pt, conference]{ieeeconf}  

\RequirePackage{amsmath}
\usepackage{amsfonts}

\RequirePackage{color}

\RequirePackage{graphicx}

\RequirePackage{algorithm}
\RequirePackage{algorithmic}

\RequirePackage{listings}

\RequirePackage{tikz}

\newtheorem{theorem}{Theorem}
\newtheorem{proposition}[theorem]{Proposition}
\newtheorem{corollary}[theorem]{Corollary}

\newtheorem{assumption}[theorem]{Assumption}
\newtheorem{lemma}[theorem]{Lemma}

\usetikzlibrary{shapes.geometric, arrows.meta}

\IEEEoverridecommandlockouts                              

\title{\LARGE \bf
Adaptive Robust Controller for handling Unknown Uncertainty of Robotic Manipulators
}

\author{Mohamed Abdelwahab$^{1}$, Giulio Giacomuzzo$^{1}$, Alberto Dalla Libera$^{1}$, Ruggero Carli$^{1}$\thanks{$^{1}$Department of Information Engineering, Università di Padova, Italy \newline {\tt\small \{abdelwahab,giacomuzzo,dallaliber,carlirug\}\newline
@dei.unipd.it}}}

\begin{document}

\maketitle
\thispagestyle{empty}
\pagestyle{empty}

\begin{abstract}

The ability to achieve precise and smooth trajectory tracking is crucial for ensuring the successful execution of various tasks involving robotic manipulators. State-of-the-art techniques require accurate mathematical models of the robot dynamics, and robustness to model uncertainties is achieved by relying on precise bounds on the model mismatch. In this paper, we propose a novel adaptive robust feedback linearization scheme able to compensate for model uncertainties without any a-priori knowledge on them, and we provide a theoretical proof of convergence under mild assumptions. We evaluate the method on a simulated RR robot. First, we consider a nominal model with known model mismatch, which allows us to compare our strategy with state-of-the-art uncertainty-aware methods. Second, we implement the proposed control law in combination with a learned model, for which uncertainty bounds are not available. Results show that our method leads to performance comparable to uncertainty-aware methods while requiring less prior knowledge.  

\end{abstract}


\section{Introduction}

In modern industry, there are increasing demands for industrial robots in several fields, ranging from welding, spraying to assembly, handling, transportation and precise manufacturing, just to mention a few. The tasks that industrial robots are expected to accomplish are becoming increasingly complex, leading to higher requirements for robots and controllers. Therefore designing high accuracy control of industrial robots is still a challenging task.

In this paper we focus on the trajectory tracking control problem for industrial robots, that is, the design of controllers able to drive an industrial robot to track a desired trajectory within a pre-assigned tolerance.  In this domain, various approaches have been investigated, spanning from traditional methodologies to sophisticated techniques. These include PID control \cite{rocco1996stability}-\cite{soriano2020pd}, feedback linearization \cite{abbas2015feedback},\cite{kali2018optimal}, adaptive backstepping control \cite{hu2012adaptive}, discrete control \cite{danik2022symbolic}, adaptive control \cite{safaei2017adaptive},\cite{slotine1988adaptive}, \cite{tian2019adaptive}, robust control \cite{sage1999robust}, sliding mode control \cite{ISLAM20112444}, neural network control \cite{jin2018robot}, fuzzy control \cite{naik2015}, iterative learning control \cite{bouakrif2016trajectory} and reiforcement learning \cite{amadio2022model}.

Here we focus on feedback linearization control or inverse dynamics control, that has shown remarkable performance when an accurate model of the system is available. The feedback linearization scheme is composed by the cascade of two loops; the outer loop allows to eliminate the nonlinearities of the system reducing the dynamics to be controlled to a double integrator that can be stabilized within the inner loop by combining a feedforward term with a PD controller. Under the assumption of perfect knowledge of the dynamics of the systems, this scheme is shown to asymptotically attain perfect tracking with transient performance that can be regulated by proportional and derivate gains.

\begin{figure}[t]
    \centering
    \includegraphics[width=\columnwidth]{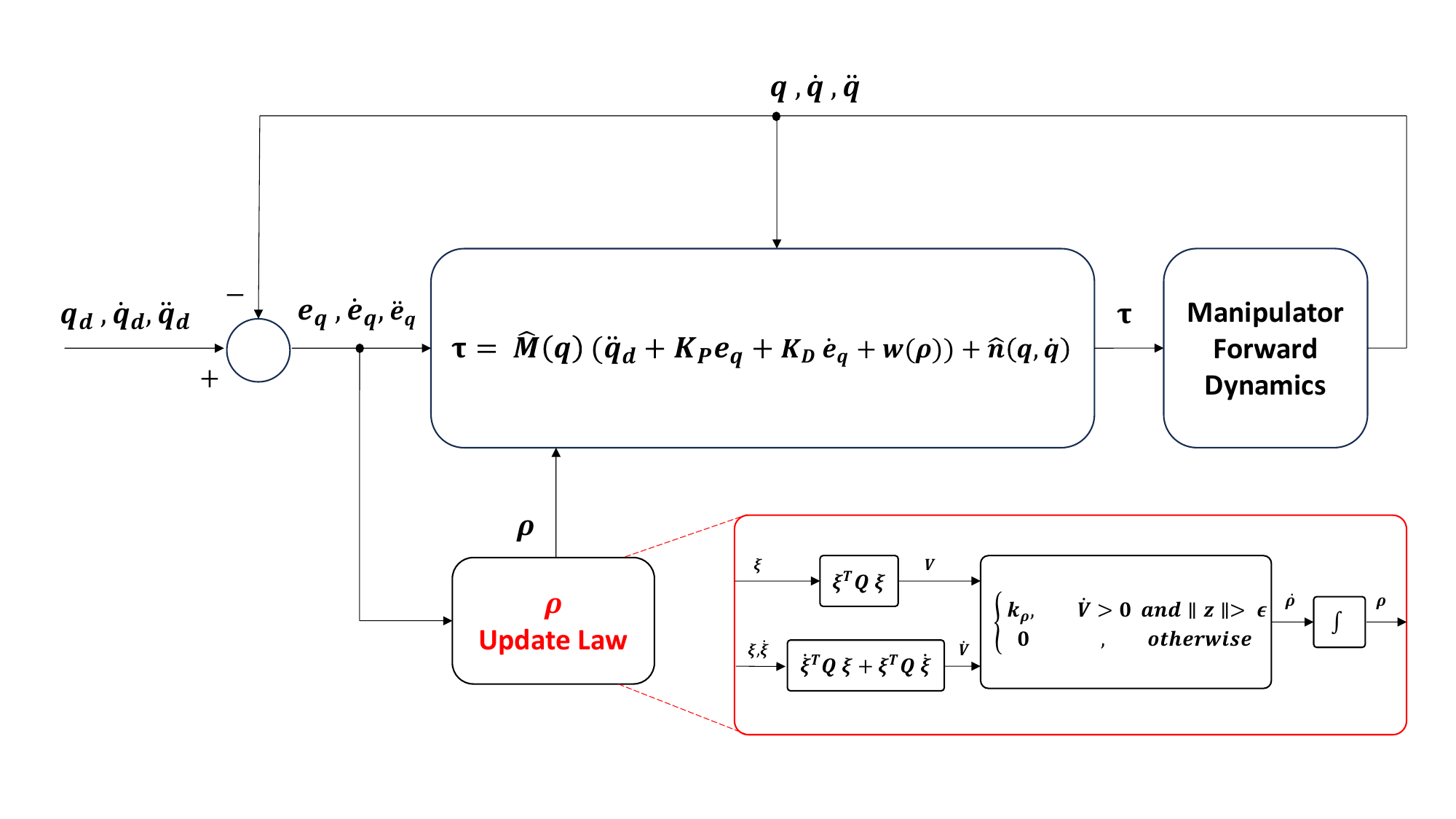}
    \caption{Controller scheme of the proposed Adaptive Robust Feedback Linearization (ARFBL)  }
    \label{fig:ARFBL scheme}
\end{figure}

In presence of model mismatch the above scheme can be robustified by introducing an additional term within the inner loop that can be designed resorting to the direct Lyapunov method, provided that bounds on the sizes of the model uncertainties are a-priori known, see section 8.5.3 of \cite{siciliano2010robotics}.
When these bounds are not available, the design of a robust feedback linearization scheme is still an active research area.

An interesting and important situation where the bounds are not available is in the context of black-box dynamics identification. Learning the robot dynamics directly from experimental data, in a black-box fashion, has gained an increasing interest over the recent years, due to the concurrent increase in robot complexity and developments in learning techniques. Black-box methods have shown promising results in approximating even complex, non-linear dynamics \cite{lutter2019deep, evangelisi-physically-consistent, ADL_GIP19,giacomuzzo2023black}, but it is particularly challenging to derive the uncertainty bounds on the dynamics components as required by the state-of-the-art robust feedback linearization approaches.     

The main contribution of this paper is to propose a novel adaptive robust feedback linearization scheme able to compensate for model uncertainties without any a-priori knowledge on them. 
The main idea is to include into the inner loop a compensating term whose size is modified according to the behavior of a Lyapunov function that would have a negative time derivative in all the points if no uncertainties were present. Loosely speaking, the rationale of the updating rule we propose is the following:  if the time derivative of the Lyapunov function is positive, meaning that there is a model mismatch which is not sufficiently compensated, then the size of the compensating term is augmented in order to increase the robustness of the control scheme. A schematic representation of the proposed strategy is reported in Fig.~\ref{fig:ARFBL scheme}. 
A theoretical characterization of the control architecture we propose is provided together with numerical results obtained running simulations on a 2 degrees of freedom (DOF) RR planar robotic arm in two different scenarios. In the first scenario, we perturbed the nominal model of the robot in such a way to have theoretical bounds on the size of the model mismatch. This allowed us to compare our strategy with the standard robust feedback linearization scheme proposed in section 8.5.3 of \cite{siciliano2010robotics}. Instead in the second scenario the dynamical model was derived adopting a black-box procedure, specifically the Gaussian Process framework with Radial Basis Function (RBF) kernel, see \cite{giacomuzzo2022advantages}. In both the considered scenarios our scheme confirmed its effectiveness in compensating unknown model uncertainties.

The paper is structured as follows. In Section \ref{sec:Prob_For} we formulate the problem of interest. In Section \ref{sec:ARFLC} we describe the novel adaptive robust feedback linearization scheme we develop together with a theoretical analysis of its convergence properties. In Section \ref{sec:black-box} we briefly discuss architectures recently proposed for learning robot dynamics directly from data. In Section \ref{sec:NumericalResults} we illustrate the numerical results we obtained. Finally in Section \ref{sec:conclusions} we gather our conclusions.

\section{Problem Formulation}\label{sec:Prob_For}

In this paper, we focus on a class of Lagrangian systems whose dynamics can be described by
\begin{equation}\label{eq:Dynamics}
M(q(t)) \ddot{q}(t) + C(q(t), \dot{q}(t)) \dot{q}(t) + g(q(t)) = \tau(t),
\end{equation}
where
\begin{itemize}
    \item $q=[q_1,\ldots, q_N]^T$ denote the $N$-th dimensional vector of generalized coordinates where $q_i$ can be either a displacement or an angle; accordingly, $\dot{q}=[\dot{q}_1,\ldots, \dot{q}_N]^T$ and $\ddot{q}=[\ddot{q}_1,\ldots, \ddot{q}_N]^T$ are the vectors of the generalized velocities and accelerations, respectively;
    \item $M(q(t)) \in \mathbb{R}^{N \times N}$ represents the inertia matrix which is a definite positive matrix for any $q(t)$;
    \item $C(q(t), \dot{q}(t)) \in \mathbb{R}^{N \times N}$ depends on the Coriolis and centrifugal forces;
    \item $g(q(t))\in \mathbb{R}^N$ describes the gravitational effects;
    \item $\tau=[\tau_1,\ldots, \tau_N]^T$ represents the vector of the generalized torques applied to system.
\end{itemize}
We consider the problem of making the system positions, velocities and accelerations $\left(q(t), \dot{q}(t), \ddot{q}(t)\right)$ to track a desired trajectory $\left(q^d(t), \dot{q}^d(t), \ddot{q}^d(t)\right)$.
Since there is no risk of confusion, for the sake of notational convenience in the following we omit the dependency on time $t$ in the vectors $q$ and $q_d$ and their derivatives.  

Assuming that perfect knowledge of $M(q), C(q,\dot{q}),g(q)$ is available, a standard control architecture proposed in the literature to solve the tracking control problem is the inverse dynamics control law given by the cascade of an outer loop and an inner loop. The outer loop consists on the following feedback linearization step 
\begin{equation}\label{eq:Feedback_Lin}
\tau = M(q) a + n(q, \dot{q}), 
\end{equation}
where $n(q, \dot{q})=C(q,\dot{q}) \dot{q} + g(q)$ and
where $a$ is an auxiliary control. The control in \eqref{eq:Feedback_Lin} allows to eliminate the nonlinearities of the system. Indeed, by substituting \eqref{eq:Feedback_Lin} into \eqref{eq:Dynamics} we obtain the double integrator dynamics 
$$
\ddot{q}=a,
$$
which can be asymptotically stabilized by designing, within the inner loop, the auxiliary input $a$ as the sum of a feedforward term and a PD controller, that is,
\begin{equation}\label{eq:auxiliary_law}
a = \ddot{q}_d + K_P e + K_D \dot{e}, 
\end{equation}
where $e=q_d - q$ is the position error, $\dot{e}=\dot{q}_d -\dot{q}$ is the velocity error, and $K_P, K_D$ are, respectively, the proportional and derivative gain matrices. 

By replacing \eqref{eq:auxiliary_law} into \eqref{eq:Feedback_Lin} we obtain the overall control scheme
\begin{equation}\label{eq:Fed_Lin_Scheme}
\tau = M(q)(\ddot{q}_d + K_P e + K_D \dot{e}) + n(q, \dot{q}).
\end{equation}
It is straightforward to see that, by applying \eqref{eq:Fed_Lin_Scheme}, the dynamics of the position error are described by the following homogeneous second-order differential equation 
$$
\ddot{e} + K_D \dot{e} +K_P e=0.
$$
It turns out that, under the assumption that $K_P,K_D$ are both definite positive matrices, i.e., $K_P,K_D>0$, $e$ converges exponentially to zero independently from the initial condition $e(0), \dot{e}(0)$.

The above control scheme strongly relies on the assumption that a perfect knowledge of the model is available. In practice this assumption is never satisfied and model uncertainty is unavoidable. From now on, let us assume that only estimates $\widehat{M}$, $\widehat{C}$, $\widehat{g}$ of $M,C,g$ are available, where typically $\widehat{M} \neq M$, $\widehat{C}\neq C$, $\widehat{g}\neq g$. In this case
by applying the control law \eqref{eq:Fed_Lin_Scheme} using the available information, that is,
\begin{equation}\label{eq:Fed_Lin_Scheme_uncertaint}
\tau = \widehat{M}(q)(\ddot{q}_d + K_P e + K_D \dot{e}) + \widehat{n}(q, \dot{q}),
\end{equation}
where $\widehat{n} = \widehat{C} \dot{q} + \widehat{g}$, one can show that the error dynamics are ruled by
$$
\ddot{e} + K_D \dot{e} +K_P e=\eta,
$$
where
\begin{equation}\label{eq:eta}
\eta=(I-M^{-1}\widehat{M}) a - M^{-1} \tilde{n},
\end{equation}
with $\tilde{n}=\widehat{n}-n$ and $a$ defined as in \eqref{eq:auxiliary_law}. In this case only a practical stability objective can be attained.

In \cite{Spong1987} and \cite{siciliano2010robotics}, in order to find control laws ensuring error convergence to zero while tracking a trajectory even in the presence of uncertainties, the auxiliary law is modified as
\begin{equation}\label{eq:modified_auxiliary_law}
a = \ddot{q}_d + K_P e + K_D \dot{e}+ w, 
\end{equation}
where the term $w$ is added to guarantee robustness to the effects of uncertainty described by $\eta$. The vector $w$ is designed using the Lyapunov direct method. Let 
$$
\xi =\begin{bmatrix}
        e  \\
        \dot e
    \end{bmatrix} \,\,\in\,\, \mathbb{R}^{2N},
$$
be the system state. Then, plugging \eqref{eq:modified_auxiliary_law}
into \eqref{eq:eta}, one can obtain, after some standard algebraic manipulations, 
\begin{equation}
\dot{\xi}= \tilde{H} \xi +D (\eta- w)
\end{equation}
where
$$
\tilde{H}=\left[
\begin{array}{cc}
0 & I \\
-K_P & -K_D
\end{array}
\right
], \qquad 
D=\left[
\begin{array}{c}
0 \\
I
\end{array}
\right
].
$$
To determine $w$, the following positive definite quadratic form is introduced as Lyapunov function candidate
$$
V(\xi)=\xi^T Q \xi,
$$
where $Q$ is a $(2N \times 2N)$ positive definite matrix \footnote{The matrix $Q$ is a block matrix of the form $Q=\left[\begin{array}{cc} Q_{11} & Q_{12} \\ Q_{12}^T & Q_{22}
\end{array}\right]$, where, for technical reasons that we do not discuss in this paper, all the blocks $Q_{11}, Q_{12}, Q_{22}$ are assumed to be full rank.}. In \cite{siciliano2010robotics}, it is proved that by adopting
$$
w=\frac{\rho}{\|z\|}z, 
$$
where $z=D^TQ\xi$ and $\rho$ is chosen in such a way that  $\rho \geq \|\eta\|$ then the time derivative of $V$, i.e., 
\begin{equation}\label{eq:V_dot}
\dot{V}(\xi)=\dot{\xi}^T Q \xi + \xi^T Q \dot{\xi}, 
\end{equation}
is less than zero for all $\xi \neq 0$, thus ensuring that $\xi \to 0$ asymptotically.

By noticing that $\eta$ in \eqref{eq:eta} is a function of $q, \dot{q}, \ddot{q}_d$, the following assumptions are made

\begin{equation}\label{eq:ass1}
\left\| I-{M}^{-1}\left( q\right) \widehat M\left( q\right) \right\| \leq \alpha \leq 1,  \quad \forall\, q,
\end{equation}
 
\begin{equation}\label{eq:ass2}
    \left\| \widetilde{n}\left( q,\dot q\right) \right\| \leq \Phi <  \infty \quad \forall \,q, \dot q,
\end{equation}

\begin{equation}\label{eq:ass3}
  \sup_{t \geq 0}  \left\| \ddot q_{d} \right\| < Q_{M} <  \infty  \quad \forall\,  \ddot q_d,
\end{equation}

\begin{equation}\label{eq:ass4}
0< M_{min} \leq \|M^{-1}(q)\| \leq M_{max} < \infty, \qquad \forall\, q.
\end{equation}
Assume that $\alpha, \Phi, Q_M, M_{max}$ are a-priori known. Then, setting 
\begin{equation}\label{eq:rho}
\rho> \frac{1}{1-\alpha} \left(\alpha Q_M+\alpha \|K\|\,\|\xi\| +M_{max} \Phi\right),
\end{equation}
where $K=\left[K_P \,\,\,K_D\right]$,
it holds that $\rho > \|\eta\|$ for all $q, \dot{q}, \ddot{q}_d$, and, in turn, $\dot{V} <0$.

As discussed in Section section 8.5.3 of \cite{siciliano2010robotics}, all the trajectories generated by the above robust control are attracted on the sliding hyperplane $z=0$ and tend towards the origin of the error state space with a time evolution governed by the size of $w$, i.e., $\rho$.

In reality, the digital controllers employed in the practical implementation impose a control signal that commutes at a finite frequency, and the trajectories oscillate around the sliding subspace with a magnitude as low as the frequency is high.
This chattering effect can be eliminated by adopting a control law of this type
\begin{equation}\label{eq:sat}
w=\left\{
\begin{array}{lcr}
\frac{\rho}{\|z\|}z & \text{if} & \|z\| \geq \epsilon\\
\frac{\rho}{\epsilon}z & \text{if} & \|z\| < \epsilon
\end{array}
\right.
\end{equation}
where $\epsilon>0$.
It is worth stressing that the above law does not guarantee that the error converges to zero, but the error is allowed to vary within a boundary layer whose thickness depends on $\epsilon$.

Now, observe that the design of $\rho$ in \eqref{eq:rho} relies on the a-priori information of the bounds on the size of the model uncertainties as described in equations in \eqref{eq:ass1}, \eqref{eq:ass2}, \eqref{eq:ass3} and \eqref{eq:ass4}. If this knowledge is not available, the design of $\rho$ is more challenging. In this paper, we aim at proposing an adaptive law for $\rho$ still ensuring that $\xi \to 0$ or $\xi$ enters within a small boundary layer, without the knowledge of $\alpha, \Phi, Q_M, M_{max}$.

\section{Adaptive robust feedback linearization control}\label{sec:ARFLC}

The method we propose is still a robust version of the feedback linearization strategy of the form
\begin{equation}\label{eq:Fed_Lin_Scheme_robust}
\tau = \hat{M}(q)\left(\ddot{q}_d + K_P e + K_D \dot{e} + \rho \frac{z}{\|z\|}\right) + \hat{n}(q, \dot{q}),
\end{equation}
with $z=D^TQ\xi$, where the value of $\rho$ is updated based on the value of $\dot{V}$ as defined in \eqref{eq:V_dot}. In particular, the value of $\rho$ is augmented when $\dot{V}$ is positive in order to compensate for the model uncertainties of the system. Formally, the updating law we propose is described as follows
\begin{equation}\label{eq:Updating_law_rho_z}
    \dot \rho = \begin{cases}
k_{\rho} & \text{if} \quad  \dot{V} \geq 0 \,\,\, \text{and}\,\,\, \|z\|\geq \epsilon \\
0   & \text{if}  \quad  \text{otherwise}
\end{cases} 
\end{equation}
where $k_{\rho} > 0 $ is a parameter tuning the increasing rate of $\rho$ and where the condition $\|z\| \geq \epsilon$ is added to avoid the chattering effect described at the end of Section \ref{sec:Prob_For}.


In the following, we provide a theoretical characterization of the law described by combining \eqref{eq:Fed_Lin_Scheme_robust} with \eqref{eq:Updating_law_rho_z} and \eqref{eq:sat}. The results we illustrate are based on the following assumptions.

\begin{assumption}\label{ass:Int-law1}
There exist $\alpha, \Phi, Q_M, M_{min}, M_{max}$ such that \eqref{eq:ass1}, \eqref{eq:ass2}, \eqref{eq:ass3} and \eqref{eq:ass4} hold true, but the values of $\alpha, \Phi, Q_M, M_{min}, M_{max}$ are not known. 
\end{assumption} 



\begin{assumption}\label{ass:Int-law2}
Consider the control law in \eqref{eq:Fed_Lin_Scheme_robust} with $\rho=0$ for all $t$. Then, the set $\Xi_0$ defined as
$$
\Xi_0=\left\{ \xi : \dot{V}(\xi)\geq 0 \right\},
$$
is a compact set.
\end{assumption}

\begin{assumption}\label{ass:Int-law3}
The function $\eta$ in \eqref{eq:eta} is a globally Lipschitz function of the time $t$.
\end{assumption}

We start our analysis with the following preliminary lemma.
\begin{lemma}
Assume Assumptions \ref{ass:Int-law1}, \ref{ass:Int-law2} and \ref{ass:Int-law3} hold true. Then,
there exists $\bar{\rho}>0$ such that if we implement \eqref{eq:Fed_Lin_Scheme_robust} with $w=\rho \,z/\|z\|$ where $\rho \geq \bar{\rho}$ then
$\dot{V}(\xi)<0$ for all $\xi \neq 0$.
\end{lemma}
\begin{proof}
In Section $8.5.3$ of \cite{siciliano2010robotics} it has been shown that when implementing \eqref{eq:Fed_Lin_Scheme_robust} then 
\begin{align*}
\dot{V} &= -\xi^T P \xi + 2 \xi^T QD(\eta-\rho \,z/\|z\|)\\
&=-\xi^T P \xi + 2 z^T(\eta-\rho \,z/\|z\|),
\end{align*}
where $P$ is a symmetric positive definite matrix such that
$$
\tilde{H}^TQ + Q \tilde{H}=-P.
$$
Now let $\rho_1 < \rho_2$ and let $\dot{V}_i(\xi)$ denote the evaluation of $\dot{V}$ on $\xi$ for $\rho=\rho_i$, $i=1,2$. From the above expression of $\dot{V}$ it easily follows that $\dot{V}_1 (\xi) > \dot{V}_2(\xi)$ for $\xi \neq 0$. 

By definition of $\Xi_0$ and from the previous reasoning, it follows that, if $\xi \notin \Xi_0$ then $\dot{V}(\xi) <0$ for any $\rho\geq 0$.

Consider now the case where $\xi \in \Xi_0$. Since $\Xi_0$ is compact, we have that there exist $\xi_M$ such that $\|\xi \| \leq \xi_M$ for all $\xi \in \Xi_0$. Then, from \eqref{eq:rho}, it follows that, if we define $\bar{\rho}$ as
$$
\bar{\rho} = \frac{1}{1-\alpha} (\alpha Q_M + \alpha \|K\| \xi_M + M_{\text{max}}\Phi)
$$
then $\dot{V}(\xi)<0$ for all $\xi \neq 0$.
\end{proof}

We have the following result.

\begin{proposition}
Let $q(0)$ and $\dot{q}(0)$ be the initial condition and let $q_d(t), \dot{q}_d(t), \ddot{q}(t)$ be the desired trajectory to be tracked. Assume Assumptions \ref{ass:Int-law1}, \ref{ass:Int-law2} and \ref{ass:Int-law3} hold true. Then, the trajectory generated by the control law \eqref{eq:Fed_Lin_Scheme_robust} combined with \eqref{eq:Updating_law_rho_z} and \eqref{eq:sat} is such that one of the following two situations is verified
\begin{itemize}
    \item[(i)] $\xi(t)$ converges to the set 
    $$S_\epsilon =\left\{\xi : \|z\| =\|D^TQ \xi\| < \epsilon\right\};$$ 
    \item[(ii)] there exists a vector $\bar{e} \in \mathbb{R}^N$ such that 
    $$
    \lim_{t \to \infty}\xi(t) = \bar{\xi} :=
    \left[
    \begin{array}{c}
\bar{e} \\
0
    \end{array}
    \right],
    $$ 
    i.e., $\dot{e} \to 0$ and $e \to \bar{e}$, and 
    $$
    \|D^TQ \bar{\xi}\|> \epsilon.
    $$
\end{itemize}
\end{proposition}
\begin{proof}
Observe that $\rho(t)$ is non decreasing and that $\rho(t) \leq \bar{\rho}$. Indeed notice that if $\rho(\bar{t})= \bar{\rho}$ at some time instant $\bar{t}$ then from Assumption \ref{ass:Int-law2} it would follow that $\dot{V}<0$ for all $t\geq \bar{t}$ and hence $\rho(t)$ would stop to increase. Moreover for $t \geq \bar{t}$ it would also hold that $\dot{V}<0$ for all $\xi$ and hence the first case would be verified.

From the previous reasoning, it turns out that 
$\rho \to \rho^*$ for some $\rho^*$. 
Assume that $\rho^* < \bar{\rho}$. Let $\Xi_{\rho^*}=\left\{\xi : \dot{V}(\xi) \geq 0\right\}$ when implementing \eqref{eq:Fed_Lin_Scheme_robust} with $\rho=\rho^*$. If $\Xi_{\rho^*} \subseteq S_\epsilon$, then again the first case is verified.

Instead suppose that $\Xi_{\rho^*}$ is not included in $S_\epsilon$. First of all observe that $\dot{V} \to 0$. Indeed, by contraction assume that $\dot{V}$ does not converge to zero. Then, this means that there exists $\bar{\epsilon}>0$ and an infinite sequence of time instants $t_{k}$, $k=0,1,2\ldots$ where $\dot{V}(t_k)> \bar{\epsilon}$. As a consequence of Assumption \ref{ass:Int-law3}, it follows that there exists an infinite amount of time in which $\dot{V}>0$ implying that $\rho^* > \bar{\rho}$ contradicting the fact that $\rho^* \leq \bar{\rho}$.

Since
$$
\dot{V}=\dot{\xi}^T Q \xi + \xi^T Q \dot{\xi}
$$
we get that $\dot{\xi} \to 0$ and hence $\dot{e} \to 0$. Therefore $e \to \bar{e}$ for some $\bar{e}$. Let $\bar{\xi}$ be defined as in the statement of the Proposition. Then, if $\bar{\xi} \in S_\epsilon$ then the first case is verified, otherwise we have that $\|D^TQ\bar{\xi}\| >\epsilon$ and hence the second case is verified.
 \end{proof}
Notice that the second situation states that the system will asymptotically track the desired trajectory with null velocity error but with a position error that is not negligible. Nevertheless, in all the numerical simulations we ran, we never encountered this situation and the error $\xi$ always converged to the set $S_\epsilon$.
However in order to get rid of this situation we modify the law in \eqref{eq:Updating_law_rho_z} adding a novel condition as follows
\begin{equation}\label{eq:Updating_law_rho_z_modified}
    \dot \rho = \begin{cases}
k_{\rho} & \text{if} \quad  \dot{V} \geq 0 \,\,\, \text{and}\,\,\,  \|z\|\geq  \epsilon \\
k_{\rho} & \text{if} \quad  -\epsilon_1 \leq \dot{V} \leq 0   \,\,\, \text{and}\,\,\, \|z\| \geq  \epsilon \\
0   & \text{if}  \quad  \text{otherwise}
\end{cases} 
\end{equation}
where $\epsilon_1>0$ is another small positive parameter. The following result characterizes the new law.

\begin{corollary}
Let $q(0)$ and $\dot{q}(0)$ be the initial condition and let $q_d(t), \dot{q}_d(t), \ddot{q}(t)$ be the desired trajectory to be tracked. Assume Assumptions \ref{ass:Int-law1}, \ref{ass:Int-law2} and \ref{ass:Int-law3}  hold true. Then, the trajectory generated by the control law \eqref{eq:Fed_Lin_Scheme_robust} combined with \eqref{eq:Updating_law_rho_z} and \eqref{eq:sat} is such that  $\xi(t)$ converges to the set $S_\epsilon$. 
\end{corollary}

\section{Black-Box Dynamics Identification}\label{sec:black-box}
The robust control strategies presented in this work rely on the knowledge of a dynamics model of the system. Learning robot dynamics from data has gained increasing interest in recent years, due to the ever-growing complexity of robotic systems and the contemporary developments of efficient learning strategies. In such scenarios, black-box techniques, which learn dynamics functions considering only experimental data, without any specific knowledge about the underlying system, have shown promising results in approximating even complex, highly non-linear behaviors. 

Combining classic and well established model-based control strategies with black-box models is an emergent and active research trend in the control community. The main difficulty is related to the fact that pure black-box techniques, while estimating the overall torque, rarely provide the inertial, Coriolis and gravity components required to implement model-based control laws such as the one presented in  Section~\ref{sec:Prob_For}. Even when this is possible, for example using the approach proposed in \cite{giacomuzzo2022advantages}, the lack of structural properties makes the components estimates unreliable, see for example the performance of the Radial Basis Function model in \cite{giacomuzzo2022advantages}. 

To overcome this limitation, an emerging trend is represented by the so-called \textit{Physics-Informed} (PI) learning methods, which propose to exploit knowledge from physics to embed structure in black-box models. In the context of inverse dynamics learning of robotic manipulators, for example, in \cite{lutter2019deep} the authors propose a novel Neural Network (NN) architecture, called Deep Lagrangian Network (DeLaN), inspired by the Lagrangian equations. In DeLan two distinct feedforward NN have been adopted to model on the one side the inertia matrix elements and on the other side the potential energy; then the overall torques are derived by applying the Lagrangian equations. This network structure allows to easily compute the required inverse dynamics components and the system energy.

The Lagrangian equations have inspired also several Gaussian Process (GP) models \cite{rasmussen2003gaussian,cheng-vector-valued-RKHS-4invDyn, evangelisi-physically-consistent, giacomuzzo2023black}. GPs have shown better data efficiency and generalization properties if compared to NN models \cite{giacomuzzo2023embedding}. Also, they provide a straightforward way to embed prior knowledge through the kernel function \cite{ADL_GIP19}. In \cite{giacomuzzo2023black}, this property is exploited to derive a multi-output Gaussian process model named Lagrangian Inspired Polynomial (LIP) estimator, which defines the kinetic and potential energy of the system as GPs and then derive the model of the torques from the Lagrangian equations. Similarly to DeLaN, the LIP model allows to easily estimate the dynamics components and the system energies. Regarding the application of GPs for control, we refer the interested reader to \cite{care2023kernel}.

The aforementioned PI methods have been successfully applied in combination with standard feedback-linearization controllers, see for example \cite{dalla2021control, lutter2019_delan_control}. However, model inaccuracies are intrinsically present in black-box methods, and in general vary with the generalization properties of the estimator, the complexity of the system and the collected data, which makes impossible to derive accurate bounds such as the ones required by the robust strategy reviewed in Section~\ref{sec:Prob_For}. This fact justifies the introduction of uncertainty-agnostic robust control strategies, such as the one proposed in this paper.



\section{Numerical results}\label{sec:NumericalResults}


In this Section we present the results obtained by evaluating the proposed ARFBL scheme on a simulated setup. In detail, we implemented a RR planar manipulator, which is a 2 DOF robot with 2 revolute joints. We simulated the robot dynamics in Python, by means of the Scipy library \cite{2020SciPy-NMeth}, with the dynamic parameters in Table~\ref{table:dyn_param}.

We designed a trajectory tracking experiment, where the reference trajectory of the $i$-th joint, hereafter denoted as $q_i^d$, is defined by the sum of random sinusoidal trajectories, namely
\begin{equation}\label{eq:ref_trj}
    q^d_i(t) = A \sum_{i=1}^{N} \sin(\omega_i t),
\end{equation}
where:
\begin{itemize}
    \item \(A\) denotes the amplitude of the sinusoids;
    \item \(N\) is the number of sinusoidal components;
    \item \(\omega_i\) is the angular frequency of the \(i\)-th sinusoid chosen randomly from a uniform distribution, in such a way that $\omega_i \sim U([\omega_{min}, \omega_{max}])$, with $\omega_{min} = \pi $ rad/s and $\omega_{max}= 3\pi $ rad/s. 
\end{itemize}


In all the considered experiments, the controller parameters in \eqref{eq:Fed_Lin_Scheme_robust} and \eqref{eq:Updating_law_rho_z}  are set such that the proportional and derivative gains are diagonal matrices defined as  $K_P = k_p\, I$ and
$K_D = k_d\, I$, with $I$ being the identity matrix and $k_p = 100$ and $k_d = 2\sqrt{k_p}$ respectively. Regarding the choice of $\epsilon$, it was taken as $0.5$, which resulted in almost no chattering.
The experiments were conducted with $K_{\rho} = 1000$  for the first experiment and $K_{\rho} = 500$ for the second experiment. 

This section is structured as follows. First, in Section~\ref{sec:perturbed_dynamics} we compare the performance of our proposed ARFBL approach with the conventional Robust Feedback Linearization (RFBL), considering a perturbed version of the dynamical model for which we can exactly compute the theoretical bounds. Then, in Section~\ref{sec:GPR} we consider a black-box model learned by means of GPR and we compare the performance of the ARFBL strategy with the one of standard feedback linearization.

\begin{table}[htbp]
    \centering
    \caption{Dynamic parameters of RR manipulator}
    \label{table:dyn_param}
    \begin{tabular}{|l|c|c|}
        \hline
        Parameter & Symbol & Value \\
        \hline
        Mass of link 1 & $m_1$ & 7.8 kg \\
        Mass of link 2 & $m_2$ & 4.5 kg \\
        Length of link 1 & $l_1$ & 0.3 m \\
        Length of link 2 & $l_2$ & 0.15 m\\
        C.O.M of link 1 & $l_{c_{1}}$ & 0.1554 \\
        C.O.M of link 2 & $l_{c_{2}}$ & 0.0341 \\
        Inertia of link 1 & $I_1$ & 0.176 kg m$^{2}$\\
        Inertia of link 2 & $I_2$ & 0.0411 kg m$^{2}$\\
        Gravitational acceleration & $g$ & 9.8 m s$^{-2}$ \\
        \hline
    \end{tabular}
\end{table}

\subsection{First Experiment: Perturbed dynamics } \label{sec:perturbed_dynamics}

As a first experiment, we considered a nominal model of the RR robot, obtained by adding a small perturbation to the inertia, lengths, and masses of each link. In this scenario, it is possible to compute exactly the bounds reported in \eqref{eq:ass1} and \eqref{eq:ass2}, which enables to compare the performance of the proposed ARFBL with the one obtained using the RBFL in \eqref{eq:Fed_Lin_Scheme_robust}. 

\begin{figure*}[ht]
    \centering
    \includegraphics[width=.9\textwidth]{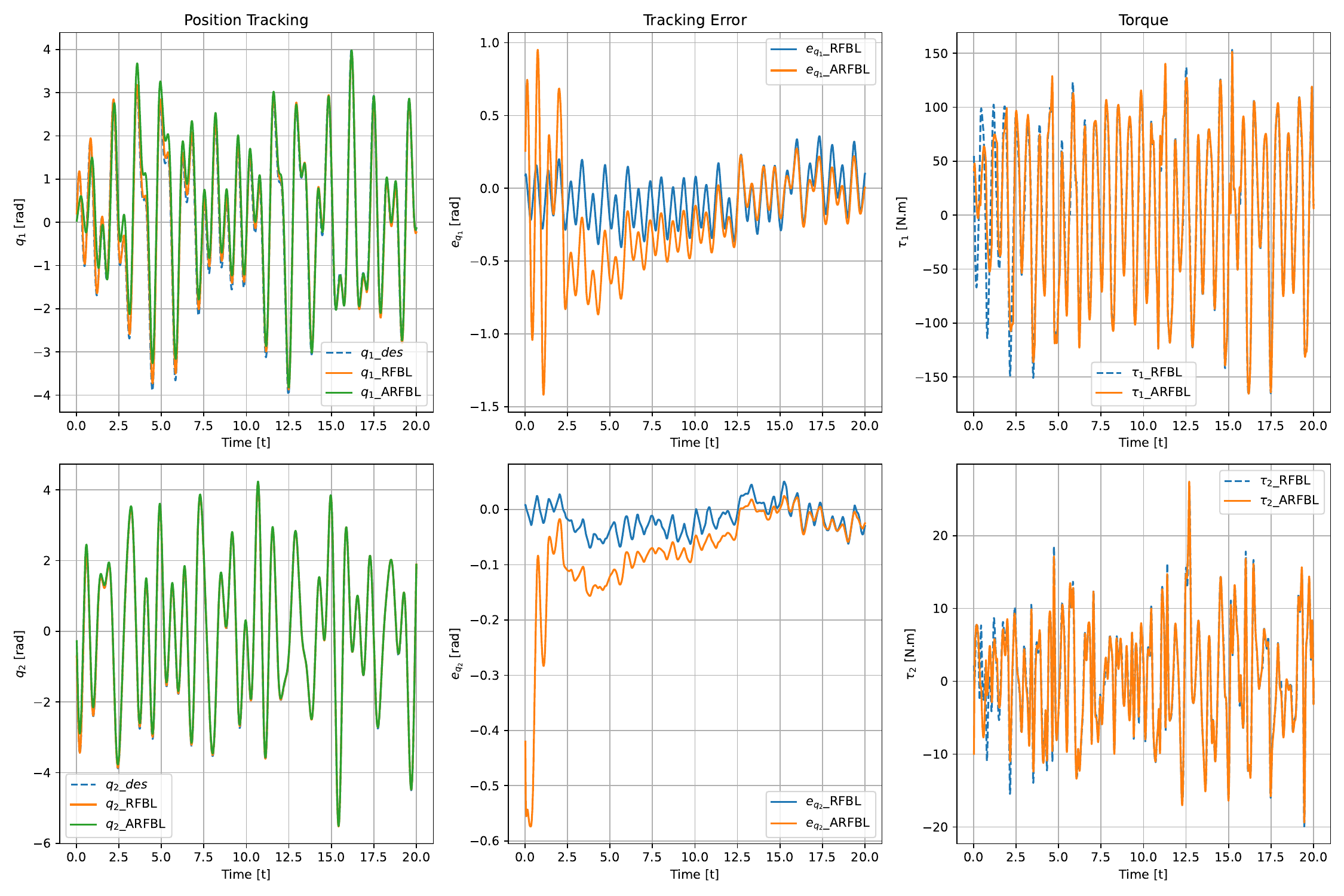}
    \caption{ Implementation results of ARFBL compared to RFBL}
    \label{fig:ARFBL_Vs_RFBL}
\end{figure*}

\begin{figure}[ht]
    \centering
    \includegraphics[width=\columnwidth]{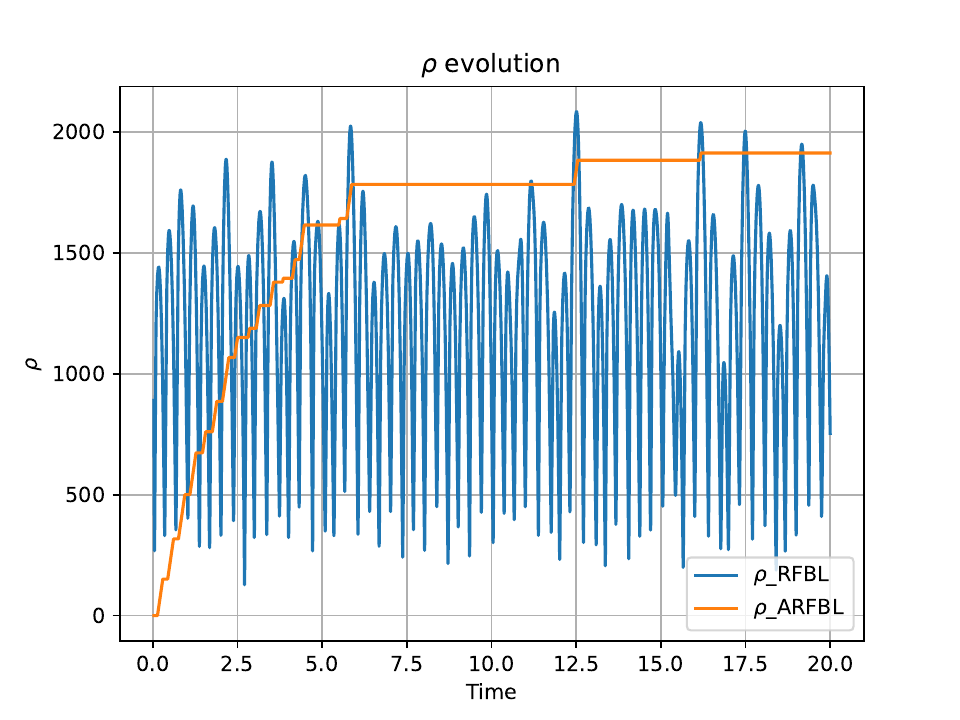}
    \caption{$\rho$ evolution of the proposed update method Vs True $\rho $ from RFBL }
    \label{fig:rho_evolution}
\end{figure}

In Fig.~\ref{fig:ARFBL_Vs_RFBL} we report the comparison in terms of joint position, tracking error, and actuation torque, while in Fig.~\ref{fig:rho_evolution} we plot the evolution of $\rho$. Results show that the ARFBL scheme converges to the same error dynamics of the standard RBFL, despite being independent of a priori knowledge concerning the uncertainty bounds. Regarding the evolution of the robust term, instead, ARFBL shows the capability of regulating the parameter $\rho$ effectively, maintaining its value within a proximity of the one derived from the uncertainty bounds. 


\subsection{Second Experiment: black-box dynamics} \label{sec:GPR}
In the second experiment, we considered a black-box dynamics model based on GPR, for which it is not possible to derive the bound in \eqref{eq:ass1}.

In order to obtain the model, we considered each joint torque $\tau_i$ in \eqref{eq:Dynamics} to be independent of the others and we learned it with a dedicated GPR model. For each torque component, we collected a dataset $\mathcal{D}_i = \{(x^t, y_i^t)\}$, with $x^t = (q(t), \dot{q}(t), \ddot{q}(t))$ and $y_i^t = \tau_i(t)$. We generated the dataset by evaluating \eqref{eq:Dynamics} on the true model, with trajectories of the same type of \eqref{eq:ref_trj}. We simulated the trajectories for $50$ seconds and we collected the data with a sampling frequency of $1$ Hertz, which resulted in datasets composed of $50$ samples. Regarding the kernel of the GPs, we used a standard RBF, whose hyperparameters are trained by means of marginal likelihood optimization, see \cite{rasmussen2003gaussian}. 
Moreover, when implementing the ARFBL and FBL strategies, we derived the estimate of the different dynamics components, namely the inertia matrix and the Coriolis and gravity torques, from the overall torque estimate using the strategy described in \cite{giacomuzzo2022advantages}.

In Fig.~ (\ref{fig:GP_experiment}) we compare the performance of the ARBFL with those of the standard FBL in terms of joint position, tracking error, and actuation torque. While the standard FBL fails the tracking task, leading to high tracking errors, ARFBL shows a notable ability to counteract uncertainties, achieving superior tracking accuracy with reduced error, and generating torque trajectories that are both smoother and of lower magnitude.
Finally, in Fig.~\ref{fig:rho_evolution_GP} we report the evolution of $\rho$. We observe that $\rho$ increases in all input regions where the GP dynamics fails to cancel the non-linearity properly, and stabilizes at a constant value upon the error's entrance into the boundary region defined in \eqref{eq:sat}.


\begin{figure*}[ht]
    \centering
    \includegraphics[width=.9\textwidth]{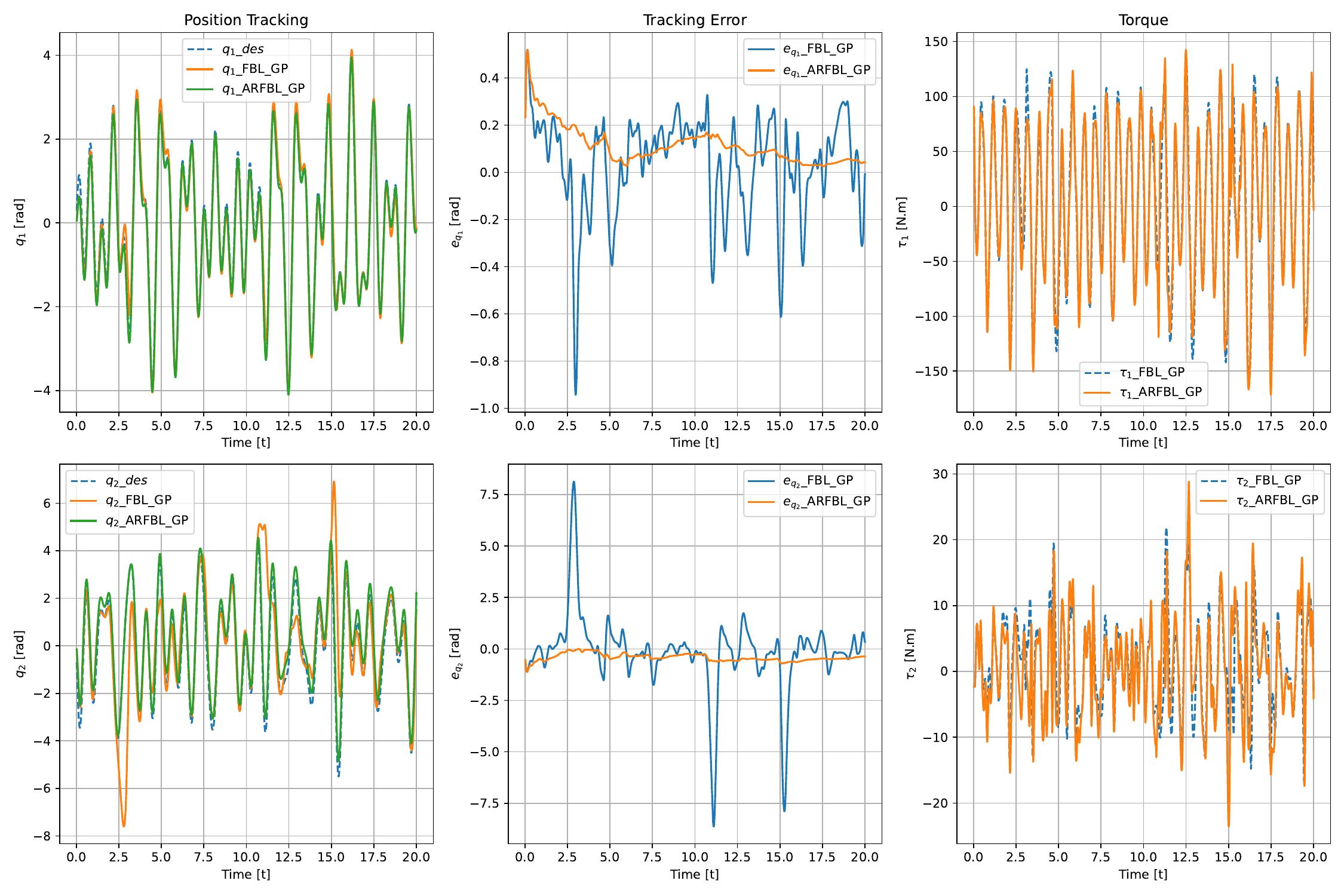}
    \caption{ Feedback linearization Vs. Adaptive robust feedback linearization using GP estimated dynamics}
    \label{fig:GP_experiment}
\end{figure*}

\begin{figure}[ht]
    \centering
    \includegraphics[width=\columnwidth]{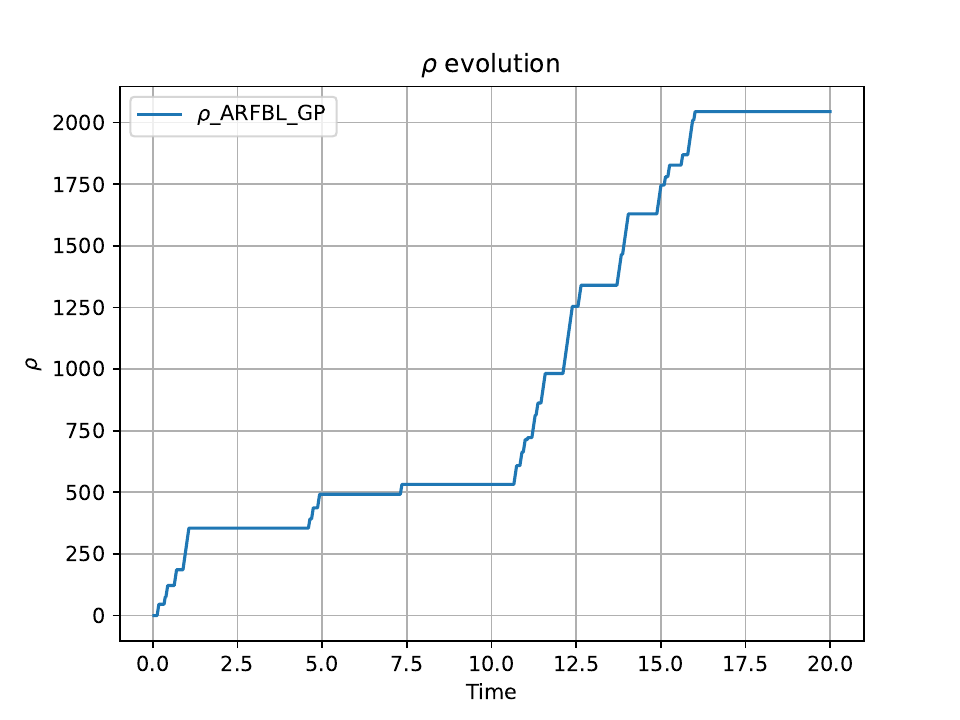}
    \caption{$\rho$ evolution using ARFBL when modeling dynamics using GPR }
    \label{fig:rho_evolution_GP}
\end{figure}
As expected, achieving only a bounded error is attributed to two primary reasons. Firstly, the use of estimated dynamics leads to only partial cancellation of the nonlinearities. Consequently, A common remedial strategy involves augmenting the Proportional-Derivative (PD) control gains, thereby enabling the linear control component to mitigate the effects of the un-canceled nonlinear term. Our proposed methods mimic a similar corrective measure adaptively, eliminating the necessity for manual gain adjustment. Remarkably, even in the presence of low gains, the update laws naturally prompt an increase in $\rho$,  effectively mirroring the impact of heightened PD gains. This adjustment comes with the added benefits of generating smoother and lower magnitude torque trajectories.

\section{Conclusions}\label{sec:conclusions}

In this paper, we introduce an adaptive robust feedback linearization scheme able
to compensate for model uncertainties without any a-priori knowledge on them, and we proved, under mild assumption, its convergence.
We experimentally evaluated the proposed method on a 2 DOF RR robot in simulation. Results show that i) compared to state-of-the-art uncertainty-aware methods, the proposed ARFBL strategy leads to comparable results at steady state, without requiring any a-priori knowledge on the model mismatch; ii) the proposed method provides satisfiable tracking performance, enabling the implementation of robust trajectory tracking when uncertainty bounds are not available.
As a future work, we plan to extend the update law to allow $\rho$ to decrease, in order to achieve a dynamics more similar to the one of the standard RFBL.





\bibliographystyle{IEEEtran}
\bibliography{ref}

\end{document}